\documentclass[sigconf]{acmart}
\usepackage[english]{babel}
\usepackage{booktabs} 
\usepackage[vlined,ruled]{algorithm2e}
\usepackage{graphicx}
\usepackage{tabu}
\usepackage{microtype}
\usepackage{subcaption}
\usepackage{algpseudocode}
\newcommand{\id}[1]{\mbox{$\mathsf{#1}$}}
\renewcommand{\vec}[1]{\mathbf{#1}}
\newcommand{\set}[1]{\mathbf{#1}}
\newcommand{\dur}[1]{\tau(#1)}
\newcommand{\R}{T_s}
\newcommand{\D}{\tau}
\newcommand{\pattern}{feature}
\DeclareMathOperator*{\argmin}{arg\,min}
\DeclareMathOperator*{\argmax}{arg\,max}
\newcommand{\spaceprint}{\emph{Spaceprint}}
\newcommand{\densitybased}{\emph{density-based}}
\newcommand{\ts}[1]{\mbox{\scriptsize\emph{#1}}}

\begin{document}
	\sloppy
\title{\textit{Spaceprint}: a Mobility-based Fingerprinting Scheme 
	for Public Spaces}

\author{Mitra Baratchi, Geert Heijenk, Maarten van Steen}
\affiliation{%
  \institution{University of Twente}
  \city{Enschede} 
  \state{The Netherlands} 
}
\email{{m.baratchi, geert.heijenk, m.r.vansteen}@utwente.nl}

\begin{abstract}

In this paper, we address the problem of how automated situation-awareness can be achieved by learning real-world situations from ubiquitously generated mobility data.
Without semantic input about the time and space where situations take place, this turns out to be a fundamental challenging problem.
Uncertainties also introduce technical challenges when data is generated in irregular time intervals, being mixed with noise, and errors.

Purely relying on temporal patterns observable in mobility data, in this paper, 
we propose \textit{Spaceprint}, a fully automated algorithm for finding the 
repetitive pattern of similar situations in spaces. We evaluate this technique 
by showing how the latent variables describing the category, and the actual 
identity of a space can be discovered from the extracted 
situation patterns.
Doing so, we use different real-world mobility datasets with data about the presence of mobile entities in a variety of spaces.
We also evaluate the performance of this technique by showing its robustness against uncertainties.

\end{abstract}


\maketitle
\section{Introduction}

Many situational-aware decision-support systems rely on the capability of describing the situation in spaces.
Ideally, these descriptions are updated automatically as the situation in a space changes.
Typical examples include automatically identifying a bottleneck on a road, or a suspicious activity in an airport.
A means for learning and comparing situations from the abundance of 
ubiquitously generated mobility data (GPS coordinates, check-in records, WiFi 
detections, etc.)  can open the door to many applications that require such 
automated 
situational-awareness.
As a first step towards this goal, in this paper, we investigate how mobility data can represent the repetitive pattern of situations in spaces.

In many cases, a specific space with a known category such as a library, a 
canteen, or a classroom, will exhibit repetitive visiting patterns 
characterizing a recurring situation.
Such patterns effectively operate as a spatial \emph{fingerprint} of situations.
Moreover, we can expect that similar places will often have similar fingerprints.
Although in many cases these fingerprints would seem to be static, it is really the usage of a space that determines its meaning, which at various occasions may differ from the location's original intended purpose.
For example, in special situations an office space is used for throwing a party or, likewise, an apartment can be rented out as if it were a hotel room.
We argue that to better understand or reason about the situation at hand, it is important to understand to what extent the situation in a space adheres to its regular fingerprint, and otherwise, to what extent it resembles any other well-known fingerprints.

In this light, we address the question to what extent we can automatically \emph{measure} a location's fingerprint of situations from available mobility data.
To realize situation-aware systems that are generally applicable, we focus on creating these fingerprints in a completely unsupervised manner.
This implies that these fingerprints should be created from raw mobility data 
without additional human input of any kind.
Therefore, unlike most previous related research in mobility data analysis, our 
method operates without a feature-engineering phase.

To this end, we study the presence pattern of devices by looking at when and how long they appear in a space.
More specifically, we make the following contributions.
(1) We propose a feature set that can generically characterize all possible presence patterns in a space.
(2) We use such a feature set to extract the fingerprint of the repetitive 
situations in spaces (\spaceprint s) in a fully unsupervised manner.
(3) We evaluate the robustness of this fingerprinting scheme in the presence of common sources of uncertainty in ubiquitously collected mobility data sets.
(4) We validate our method by showing its classification performance using a 
WiFi-based detection data set and a Foursquare check-in dataset.

The rest of this paper is organized as follows.
Related work is presented in Section~\ref{sec:related}.
We present our problem definition and a sketch of our proposed fingerprinting 
framework in Sections~\ref{sec:system-model} and \ref{sec:framework}, 
respectively.
The details of our fingerprinting scheme is presented in Section~\ref{sec:methodology}.
The performance of this scheme is evaluated in Section~\ref{sec:evaluation}.
A number of remarks conclude this paper in Section~\ref{sec:conclusion}.

\section{Related work}
\label{sec:related}
There are two ways to study the movement of individuals in a space when dealing 
with mobility data (referred to as the Lagrangian and Eulerian 
approaches~\cite{BaratchiMeratniaHavingaEtAl2013}).
First, from the perspective of an individual, one may ask about the whereabouts of a person: what are the locations that someone visits?
When do those visits take place, and for how long?
The research in this direction concentrates on extracting mobility patterns 
that reflect an \emph{individual mobility fingerprint} for frequent 
behavior \cite{Lian:2015:CCE:2745393.2629557,Giannotti:2007}, periodic behavior 
\cite{Li:2012:MEP:2339530.2339604}, social behavior 
\cite{wang2011human}, etc.

Second, from the perspective of a specific location, 
one may ask about the visits to that location: When do they take place?
How long do they last?
Which visits happen again? In this case, one focuses on extracting a 
\emph{spatial mobility fingerprint}. 
 Previous related research in 
extracting these spatial fingerprints have either focused on improving the 
individual mobility prediction models
\cite{Wang:2015:RCL:2783258,gao2012mobile} or on 
bringing 
sense to 
raw location coordinates in terms of 
meaningful 
labels. Research in methods to describe the meaning of locations, primarily 
concentrates on
how 
accurately trajectories can be segmented into sections with basic semantics 
such 
as 
\textit{stop and move} areas \cite{spaccapietra2008conceptual}, or 
\textit{points of 
interest} \cite{Montoliu2013}. With the prevalence of context-aware mobile  
applications which needed more than just such low-level semantics, further 
research has been performed to extract more detailed semantics about 
spaces interpreted in colloquial terms such as \textit{home}, 
\textit{work}, 
\textit{cinema},
\textit{restaurant}, etc. Using a single person's frequent trajectory patterns, 
 semantics about few predefined places (e.g. 
\textit{home}, \textit{work}) have been extracted in  
\cite{lv2016discovery,Eagle2009}. In a more general case, and when extracting 
semantics from data 
involving a large population of 
mobile entities, a common 
approach
has 
been enriching data 
with higher 
level 
semantics using additional sources, or using common sense assumptions, for 
instance, 
presence at night for \textit{home}, presence at working hours 
for 
\textit{offices} or 
presence in weekends for \textit{leisure related locations}. Some examples of 
additional sources 
of 
semantics are verbal terms used 
by 
people in 
social media such as twitter \cite{7026273,Wu:2015:SSA:2740908.2742837},
and third party geographical sources 
\cite{Yan:2013}. In
 \cite{6427751} the authors use a number of selected mobility 
features (e.g., crowded hours, number of visitors per month) 
along with 
application usage, and proximity to other devices
to 
label a 
group of known spaces.
Knowing the semantic labels of spaces within a region, higher level regional 
semantics 
have also been extracted to label 
\emph{areas} such as those used for \textit{housing}, and \textit{businesses} 
\cite{yuan2015discovering}. 

The spatial fingerprints made thus far are either meant for labeling 
location 
coordinates using \emph{engineered} features in a supervised manner or use 
additional 
\textit{semantic 
input} 
to enrich data with context from other 
sources. 
 These 
approaches are not generic and cannot be taken further to realize automated 
situation-awareness in 
dynamically changing spaces purely using mobility data. To reach 
this goal, 
our approach in spatial fingerprinting from mobility data is different 
from all 
previous research as it performs in a fully unsupervised manner purely 
exploiting 
presence patterns in spaces. Specifically, instead of looking for features that 
characterize 
spaces based on their 
\textit{semantic meaning}, we look for features that can characterize
periods of time in a space based on its \textit{dynamic situation}.


\section{Problem definition}
\label{sec:system-model}

We define a model based on data acquired from any system that allows for the 
collection of mobility data in terms of presence or detection of mobile 
entities in a well-defined region of space.
A \textbf{detection} record is a tuple $\langle \id{d}, \id{s}, t \rangle$ with \id{d} being the identifier of the detected mobile entity, \id{s} being the identifier of the space where the entity \id{d} was detected, and $t$ being the timestamp of the detection.
A variety of mobility-data collection systems can result in such a dataset.
These include, for example, WiFi detection of mobile devices near access points, GPS coordinates discretized in grid maps, and check-in records collected from location-based social networks.

Given a set of detection records $\set{DT}$, we are interested in a \textbf{spatial fingerprint} $\set{SP}(\id{s})$ which defines the core repeating temporal presence patterns of space $\id{s}$.
Assuming that \textit{latent} variables such as the unique \textit{identity} of 
the space and its semantic \textit{category} result in such a fingerprint, we 
demand that this scheme exhibit the following:
(1) each space has a unique fingerprint,
(2)~spaces having the same category have similar fingerprints, and
(3)~spaces having different categories have different fingerprints.

\section{Framework overview}
\label{sec:framework}

Our goal is to define a spatial fingerprint that summarizes the situations in a space in terms of repeating presence patterns over time.
One might think of creating a time series by measuring a feature from the detections over equally sized duration windows with a specific resolution, such as the number of detections (\emph{feature}) during every hour (\emph{resolution}) of a day (\emph{duration}).
By averaging the value of these features over many duration windows (e.g., over 100 days), the fingerprint can be extracted.
If these features were enough to fingerprint a space, with a suitable classification algorithm and suitable distance function, we would also be able to classify different spaces from one another based on their fingerprint.
However, there are many unknown factors that require attention.
The challenge in our case is to identify (1) the \textbf{features}, (2) an appropriate \textbf{resolution} and \textbf{duration window}, and (3) a suitable \textbf{distance function}.
Compared to these three, the choice of a classification or clustering algorithm is a trivial one.
Typically, these challenges are addressed based on intuition.
For instance, we may assume that a weekly pattern governs the visits to a space 
or that a resolution of one hour is enough to provide the necessary level of 
detail.
This intuitive approach, however, limits the applicability of the fingerprinting scheme.
The proposed fingerprinting scheme in this paper addresses these challenges through systematically finding appropriate parameter settings in an unsupervised manner.
We define a spatial fingerprint as follows.
\begin{definition}\textit{{(Spatial fingerprint)}
    \label{def:fingerprint}
    The \textbf{fingerprint for the space} $\id{s}$ is a triplet $\mathbf{SP}(\id{s})=\langle \vec{V},FD,FR \rangle$, with \textbf{feature vector} $\vec{V}=[v_1, \ldots, v_n]$, of which each element $v_i$ represents the value of a specific feature.
    $FD$ is the \textbf{fingerprint duration}, indicating the total time over which the fingerprint is configured.
    $FR$ is the \textbf{fingerprint resolution}, indicating the minimum time interval over which detections are sampled to extract features.
    $FD$ is a multiple of $FR$: $\exists r \in \mathbb{N}: FD = r \cdot FR$.}
\end{definition}
 
Algorithm \ref{algorithm1} summarizes the fingerprinting framework \emph{Spaceprint} proposed in this paper.
Let $\set{DT}$ denote a set of detections and $t_{\ts{min}}(\set{DT}) = \min \{t | \langle \id{d}, \id{s}, t \rangle \in \set{DT}\}$, i.e., the timestamp of the first detection.
Likewise, we have $t_{\ts{max}}(\set{DT})$ for the timestamp of the last 
detection and $\dur{\set{DT}} = t_{\ts{max}}(\set{DT})-t_{\ts{min}}(\set{DT})$ 
for the duration of collecting $\set{DT}$.
Denote by $\set{\overline{DT}}$ the set of detections $\{\langle \id{d}, \id{s}, t-t_{\ts{min}} \rangle | \langle \id{d}, \id{s}, t \rangle \in \set{DT} \}$, i.e., the set of same detections, but now transformed such that the first detection starts at time 0.
Finally, we use the notation $\set{DT}(\id{s})= \{ \langle \id{d}, \id{s}, t \rangle | \langle \id{d}, \id{s}, t \rangle \in \set{DT} \}$ to denote the set of detections by space \id{s}.

The spatial fingerprint is composed of three components.
We have a separate procedure for extracting each of these components.
We use the procedure \emph{fingerprintParameters} for calculating the optimal fingerprint parameters, being the fingerprint duration ($FD$) and fingerprint resolution ($FR$).
The procedure \emph{vectorize} constructs the feature vector over a dataset spanning a duration of $FD$ time units using resolution $FR$.
The final procedure \emph{vectorAverage} computes the average over multiple feature vectors.
In the following sections, we will represent how we create the feature vector and determine the fingerprint duration and resolution.
    
\begin{algorithm}
	
  \KwData{$\set{DT}(\id{s})$}
  \KwResult{ $\mathbf{SP}(\id{s}) =\langle \vec{V},FD,FR \rangle$}
  $(FD,FR)=\mbox{\emph{fingerprintParameters}}(\mathbf{DT}(\id{s}))$;
  
  \For {$(i=0 ;i < {\dur{\set{DT(\id{s})}}}/{FD}; i=i+1)$}{
    $\set{DT}_i = \{\langle \id{d},\id{s},t \rangle \in 
    \set{\overline{DT}}(\id{s}) | i\cdot FD \leq t < (i+1)\cdot FD \}$;
	${\vec{V}_i=\mbox{\emph{vectorize}}(\set{\overline{DT}}_i,FD,FR)}$;
  }
  
  ${\vec{V}=\mbox{\emph{vectorAverage}}(\vec{V}_{i\in 
  {1...\dur{\set{DT(\id{s})}}/FD}})}$;
  
  return ($\vec{V},FD,FR$);
  
  \caption{Spaceprint}
  \label{algorithm1}
  
\end{algorithm}

\section{Methodology}
\label{sec:methodology}

\subsection{Presence patterns}

As mentioned before, the most important step in fingerprinting spaces is identifying suitable (computable) features that represent the situation in spaces.
Let us consider selecting features that may be relevant for such purpose and are observable from mobility data.
For example, intuitively one may think of static features such as opening or closing hours, peak hours, group sizes, number of individuals, etc.
However, features that can define the situation in a space are numerous and intuitively coming up with a comprehensive set of features that could characterize any thinkable situation in spaces is practically impossible.

Without any intuitive assumptions about features that define the situation in a space, the only measurable feature from detections is related to presence pattern of mobile entities.
In reality, each space observes many of these patterns formed due to the variety of the intention of its visitors.
For instance, consider the presence pattern of shopkeepers in a shop versus that of their clients.
A shopkeeper enters the shop around opening time and leaves around closing time.
The clients may appear during opening hours and stay for some time based on their intention (browsing or shopping).
We assume that the situation in space is reflected in the overlapping visits of different groups of mobile entities.
To consider this variety, we define a presence pattern such that it reflects \textbf{the synchronous presence of a group of mobile entities during a specific period of time}.
Such a pattern represents a group of mobile entities entering a space, staying there for a specific amount of time, and then leaving it at the same time.
Entering and leaving a space may be repeated multiple times as well.
Extracting these patterns from a detection dataset can be achieved by counting the number of mobile entities in a window with a specific starting time, $t_{\ts{start}}$, and duration, $\tau$.
As detections are registered in discrete time intervals, the presence should be detected in all sampling intervals {of length $T_s$} in $\tau$.
Correspondingly, we define presence features with the following \emph{template} to quantify the intensity of such patterns.

\begin{definition}\label{def:pattern} \textit{ {(Presence \pattern)}
    A \textbf{presence \pattern} $PF(t_{\ts{start}}, \tau, T_s)$ over 
    a space represents the number of mobile entities that were detected in all 
    $\lceil \tau/T_s \rceil $ consecutive sampling intervals of length $T_s$ 
    within a 
    measurement window, starting at time $t_{\ts{start}}$ and lasting for a 
    duration of $\tau$ time units.}
\end{definition}

By ranging over all possible values of the parameters $t_{\ts{start}},\tau$, and $\R$, the feature template mentioned above will lead to numerous presence features.
Table~\ref{tbl:range} summarizes the possible range of these parameters for creating presence \pattern s as defined in Definition~\ref{def:pattern}.
\begin{table}[!ht]
  \caption{\label{tbl:range}The possible ranges for the parameters of a presence pattern, given a fingerprint duration $FD$ and fingerprint resolution $FR$.}
  \begin{center}
    \begin{tabular}{cl}\hline
      \textbf{Variable Name} & \textbf{Range} \\ \hline
      $t_{\ts{start}}$    &  $\{0 \leq k \cdot FD/FR < FD, k \in \mathbb{N} \}$ \\
      $\tau$        &  $\{0 \leq k \cdot FD/FR < FD - t_{\ts{start}}, k \in 
      \mathbb{N}\}$ \\
      $T_s$         &  $FR \ll T_s \ll  FD$ \\ \hline
    \end{tabular}
  \end{center}
\end{table}

These parameter ranges are motivated as follows.
Assume that we measure detections at a given location for a specific duration of time, $FD$, and that the mobile entities are detected at a frequency $f_p$ (and period $T_p= 1/f_p$).
For now, also assume that the fingerprinting resolution $FR$ is equal to this period as well ($T_p=FR$).
We later show how to extract the optimal value for $FR$ which is possibly bigger than $T_p$.
The basis of our approach is to sample the number of mobile entities within a specific \textbf{duration window} $W = \langle t_{\ts{start}},\tau \rangle$ with a \textbf{sampling frequency} $f_s$ (with period $T_s= 1/f_s$).
Both $W$ and $f_s$ can vary.
The duration window can have any starting time and length as long as the window is smaller than $FD$.
Therefore, we require that $\tau \leq FD$ and $t_{\ts{start}}+\tau < FD$.
To count the number of mobile entities, we need to sample detections with a period $T_s$.
Obviously, as it does not make sense to sample with a speed faster than the mobile entity's detection generation speed, we require that $T_s \geq FR$ (or $T_p$).
Additionally, $T_s$ cannot be larger than the duration window, i.e., $T_s \leq \tau$.
Note that the feature vector $\vec{V}$ can now be considered as an ordered list of normalized presence \pattern s.

\subsection{Feature vector}

As mentioned before, the presence \pattern s can be created by counting mobile entities based on every possible combination of starting time, stay duration, and sampling period, $t_{\ts{start}}, \tau, T_s$.
Considering that we have $n$ possible combinations by ranging over these parameters, we will have an $n$-dimensional vector composed from different presence \pattern s.
Algorithm~\ref{algorithm2} (\emph{vectorize}) represents the way of constructing a feature vector for a given space based on a collection of mobile entity detections.

The input of this algorithm is a set of detections $\set{DT}(\id{s})$ for a specific space \id{s}.
If $W$ is a duration window, we write $\set{DT}[W]$ to denote the subset of detections that occurred inside $W$.
If $\R$ is a sampling period, then $[\set{DT}]_{\R}$ denotes the list of $\lceil (t_{\ts{max}}(\set{DT}) - t_{\ts{min}}(\set{DT}))\rceil/{\R}$ buckets, with the $i^{th}$ bucket containing all detections that occurred during the $i^{th}$ interval of length $\R$.

The essence of \emph{vectorize} is to count the number of mobile entities that were detected during an entire duration window, $W$, when sampled with the period $T_s$.
We systematically explore every possible duration window and sampling period for a given fingerprint duration $FD$ and resolution $FR$.
There are three loops for covering all possible values for parameters 
$t_{\ts{start}}, \tau$ and $T_s$.
In each iteration, by counting the mobile entities 
that 
appeared in the intersection of
all buckets of $[\set{\overline{DT}}[W]]_{\R}$, a presence feature is created.

\begin{algorithm}
	\KwData{$\set{\overline{DT}}$, $FD$, $FR$}
	\KwResult{$\vec{V}$}
	$\vec{V}=[]$;\\
	\For {$(t_{\ts{start}}=0; t_{\ts{start}} < FD; t_{\ts{start}}=t_{\ts{start}}+FR)$}{ 
	  \tcc{iterate over all durations}	
	  \For{$(\tau = FR; \tau \leq FD - t_{\ts{start}}; \tau = \tau + FR)$}{
		\For {$(\R=FR; \R \leq \D; \R=\R+FR )$}{
		  \tcc{iterate over all sampling periods}
		  \If {$(\D\ mod\ \R=0)$}{
			$W = \langle t_{\ts{start}}, \D \rangle$\;
			$u=\bigcap([\set{\overline{DT}}[W]]_{\R})$;
			\tcc{get the ID of mobile entities present in all buckets of window 
			W}
			$append(\vec{V},count(u))$;
			\tcc{append to $\vec{V}$ the total number of mobile entities}
          }
        }
      }
    }
	return$(\vec{V}/\max(\vec{V}))$;
	\caption{vectorize}
	\label{algorithm2}
	
\end{algorithm}

The complexity of Algorithm~\ref{algorithm2} presented above is $
\mathcal{O}((\frac{FD}{FR})^4)$. 
This complexity comes from the three for loops and an intersection over all elements of $[\set{\overline{DT}}[W]]_{\R}$. 
By reusing the results of the intersection operation this complexity can be reduced to 
$\mathcal{O}((\frac{FD}{FR})^3)$.
A meaningful sampling period is the one that can break the duration window into its integer factors $(\D\ mod\ \R=0)$. 
In that case, the third loop will repeat only for integer multiples of $\D$, 
thus reducing complexity further.
It should be noted that both $FD$ and $FR$ are fixed and do not depend on the size of detection dataset.
Therefore, creating the feature vectors can be performed in a scalable manner.

Figure~\ref{vector} represents an example feature vector $\vec{V} = [v_1,\ldots,v_n]$ calculated using Algorithm~\ref{algorithm2}.
This vector is acquired by vectorizing one week of data with a resolution of a day (i.e., $FD = 7$ days and $FR = 1$ day).
It can be readily verified that there are $n = 57$ elements in $\vec{V}$.
The first element, $v_1$, corresponds to the number of mobile entities that 
were detected during the first day: $W = \langle 0,1 \rangle$, with resolution 
$\R=1$.
Element $v_2$ counts the mobile entities that were present during both the first and the second day: $W = \langle 0,2 \rangle$, with sampling period $\R=1$.
Likewise, $v_3$ represents mobile entities during the either first or second day, and so on.
In this example, $v_{15}$ represents a window spanning over the entire week ($W = \langle 0,7 \rangle$) and sampled with the sampling period $\R=1$.
It thus counts the number of mobile entities that were present in all seven 
days.
Typically, these encompass all static, that is, non-mobile entities.
Also interesting is $v_{16}$, which represents a duration window spanning over all seven days ($W = \langle 0,7 \rangle$), but with a sampling period $\R = 7$ of also the entire week.
As such, it counts the total number of mobile entities who showed up at least once during the entire week, regardless how long they stayed.

\begin{figure}
  \begin{center}
    \includegraphics[width=0.27\textwidth]{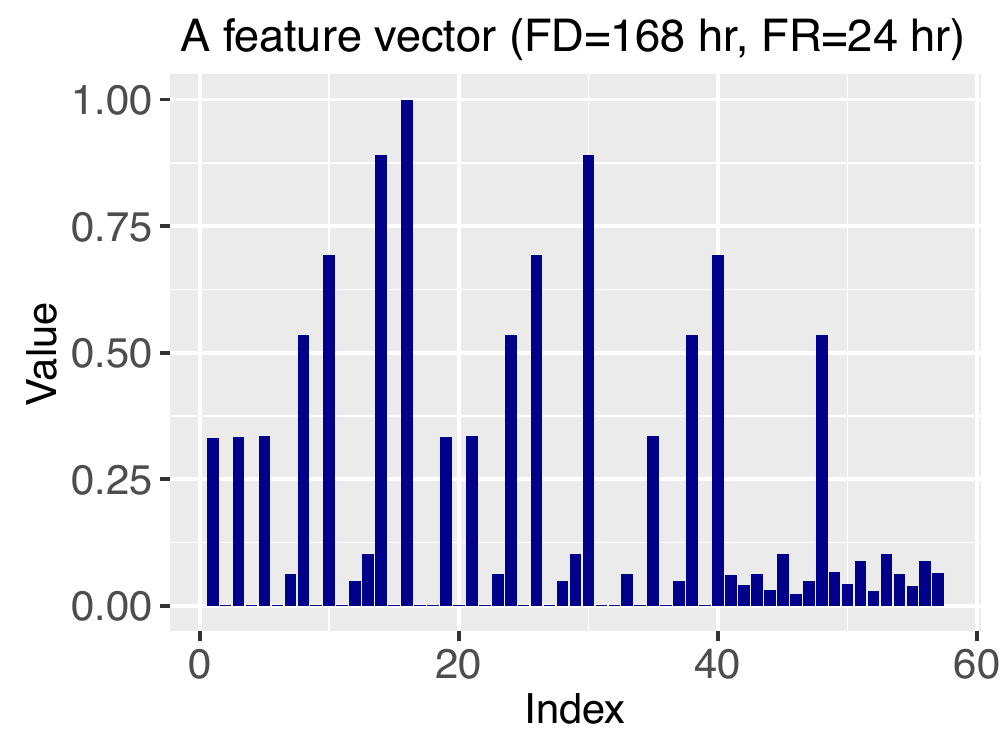}
    \caption{An example representation of a feature vector}
    \label{vector}
  \end{center}
\end{figure}

Our goal is to use such feature vectors to compare spaces to each other based on visiting patterns of devices.
In doing so, we need to take into account that the values in a single vector can vary widely, which is entirely due to the fact that we wish to include all possible values for duration windows and sampling periods into a single data structure.
As a consequence, we need to avoid that high values (which are perfectly natural due to our method of counting) dominate our perspective of difference between two vectors.
In order to take these natural differences between elements into account, we choose a distance metric based on the so-called \emph{Canberra Distance}~\cite{emran2001robustness}.
\begin{definition} \textit{(Feature vector distance function)
    \label{def:distance}
    Given two feature vectors $\vec{V}$ and $\vec{V^*}$ of equal length $n$, calculated using the same pair of fingerprint parameters $FD$ and $FR$, their mutual distance is $\Delta(\vec{V},\vec{V^*})=\frac{1}{n}\sum_{i=1}^n\frac{|v_i-v_i^*|}{|v_i|+|v_i^*|}$ }
\end{definition}

\subsection{Fingerprint parameters}

We now concentrate on finding appropriate values for the fingerprint duration $FD$ and the fingerprint resolution $FR$.
Concerning the \textbf{fingerprint duration}, note that we are looking for the period (in the formal sense) of repetitive or self-similar situation.
There are many ways of doing this, for example through Fourier analysis or computing auto\-correlations.
In our approach, we look for a series of consecutive fixed-length windows $W_1,W_2,W_3,\ldots$ such that for a given set of detections $\set{DT}$, we have a minimal accumulated distance between all possible pairs of vectorized subsets of detections $\set{DT}[W_i]$ and $\set{DT}[W_j]$.
Our only variable is the length of all such windows, and the length that 
\textit{minimizes} the accumulated distance is our fingerprint duration.

Determining the best \textbf{fingerprint resolution} is a bit trickier.
The resolution, as shown in Algorithm~\ref{algorithm2}, determines the minimum sampling period and directly determines the number of features in the vector.
Therefore, other than increasing the computational costs, a too detailed $FR$ may also introduce the problem of over-fitting.
It is desirable to choose the resolution such that all significant differences between feature vectors are preserved.
Therefore, what we are looking for is a resolution that \emph{maximizes} the distance between two vectorized datasets.
The assumption is that we have already determined the periodicity $FD$ in a series of detections.
By looking at two consecutive datasets of duration $FD$, a resolution $FR$ that maximizes the mutual distance of their vectorized versions effectively captures all differences that would have also been captured by a smaller resolution.
At the same time, such a resolution will capture more differences than any larger resolution (which would show a smaller distance between the two vectorized datasets).

Lemma~\ref{lemma:fr} tells us that such a distance-maximizing resolution 
actually exists.

\begin{lemma}\label{lemma:fr}

In case a space has a periodic fingerprint, there exists an optimal fingerprint resolution $FR$ over which the distance between consecutive feature vectors is maximized.
	
\end{lemma}

\begin{proof}
  We prove that having a constant fingerprint duration, by either increasing or decreasing the resolution, the distance between two features vectors $\vec{V}$,$\vec{V}^*$ approaches zero.
  Let $\delta_i = \frac{|v_i-v_i^*|}{|v_i|+|v_i^*|}$.
  When we increase the resolution, we will necessarily increase the length $n$ of a constructed feature vector.
  As both $v_i$ and $v_j$ have positive values, regardless the changes in $\delta_i$ when increasing $n$, we will always see that $\delta_i \leq 1$, while the number of elements for which $\delta_i > 0$ will increase to a finite number $M$.
  This is due to the fact that elements acquired with a smaller sampling period 
  ($T_s<T_p$) are meant to count the mobile entities that were detected with a 
  speed much faster than the actual detection speed of mobile entities and 
  there are hardly any of them.  As a consequence,
  \[
  \lim_{n\rightarrow \infty} \Delta(\vec{V},\vec{V^*}) = \lim_{n\rightarrow \infty} \frac{1}{n} \sum_{i=1}^n \delta_i \leq \lim_{n\rightarrow \infty} \frac{1}{n} M = 0
  \]
  Analogously, as the resolution decreases, the length of a feature vector decreases and will eventually be 1 when $FR$ = $FD$.
  A vector of length one will have only one element, which after normalization, is equal to 1.
  Therefore,
  \[
  \lim_{n\rightarrow 1} \Delta(\vec{V},\vec{V^*}) = \frac{1}{1} \frac{|1 - 1|}{|1 + 1|} = 0
  \]
\end{proof}

Algorithm \ref{algorithm3} summarizes the procedure of extracting the fingerprint parameters.

\begin{algorithm}
   \KwData{$\set{DT}(\id{s})$, $r$ (such that $FD = r \cdot FR$)}
   \KwResult{$FD$,$FR$ }
     \For {$(i =1;i<\dur{\set{DT}}/{(2r)}; i=i+1)$ }{
       $m = i \cdot r$;\\
       \For{$(j=0;j<\dur{\set{DT}}/m;j=j+1)$}{
         $\set{DT}_j = \{\langle \id{d},\id{s},t \rangle \in \set{\overline{DT}}(\id{s}) | j \cdot m \le t < (j+1) \cdot m \}$;\\
         $\vec{V}_j^i$=  vectorize($\set{\overline{DT}}_j,m,i$);
       }
     }
     $FD$= $r \cdot \argmin_{i} \sum_{j,k}\Delta(\vec{V}_j^i,\vec{V}_k^i)$;\\
     \For {$(i=1; i\leq FD ;i=i+1)$ }{
       \If {($ FD\ mod\ i =0$)}{
         \For{$(j=0;j <\dur{\set{DT}}/FD;j=j+1)$}{
	       $\set{D\hspace*{-1pt}T}_j\! =\! 
           \{\langle \id{d},\id{s},t \rangle \in \set{\overline{DT}}(\id{s}) | j \cdot FD \le t\! < (j + 1) \cdot FD \}$;\\           
	       $\vec{V}_j^i$= vectorize($\set{\overline{DT}}_j,FD,i$);
	     }
       }
     }
     $FR$= $\argmax_{i} \sum_{j,k}\Delta(\vec{V}_j^i,\vec{V}_k^i)$;\\
     return($FD$,$FR$)
       
    \caption{fingerprintParameters}
    \label{algorithm3}
\end{algorithm}

\section{Evaluation}
\label{sec:evaluation}
In this section, we show how \textit{Spaceprint} feature vectors can be used 
for finding repetitive situation patterns in spaces.
We also evaluate the performance of \spaceprint\ in presence of uncertainties.

\textbf{Evaluation approach:} We expect that the 
fingerprint of situations in
a space can reflect from which and what kind of space it is extracted. 
Therefore, 
we evaluate our method to see how the latent variables of the 
semantic category of a space and its 
unique identity are reflected in the fingerprint of the space. 
Our evaluations are on the basis of using the feature vectors 
mentioned before in unsupervised classification to infer these latent variables.
Any unsupervised classification or clustering algorithm can be used for such 
purpose.
In our experiments we have used \emph{K-means} clustering algorithm using our 
defined distance function from Definition~\ref{def:distance}.

\textbf{Baseline:} To the best of our knowledge, there is no prior work in 
classifying or creating situation fingerprints for spaces purely based on 
presence patterns.
However, a common approach in considering space-specific temporal features, is 
calculating hourly densities \cite{gao2012mobile,yang2016poisketch}.
Therefore, we compare \spaceprint\ with a \emph{density-based} approach as baseline.
The \emph{density-based} feature vectors are represented by 
$\vec{V}_d=[d_0,...,d_{{\frac{FD}{FR}}-1}]$ where each element $d_i$ represents 
the number of mobile entities appearing in the window $W =\langle i \cdot 
FR,FR \rangle $.
These vectors are extracted using the same fingerprint parameters ($FD, FR$).

\subsection{Test with synthetic dataset}

\subsubsection{Synthetic dataset generation}

Our goal of using a synthetic dataset is to test the robustness of the fingerprinting scheme against uncertainties, yet in a controlled fashion.
We proceed as follows.

\textbf{Generating virtual spaces:} First, a total of $NS$ different \emph{virtual spaces} are created with presence patterns that are repeated over $FD$ time units and mobile entities being detected with the same detection frequency ($T_p=1$).
A virtual space is characterized by a tuple $\langle \set{P}, NP \rangle$ of 
presence patterns $\set{P}$ each having size $NP$.  Complying with the 
definition of presence patterns in Section~\ref{sec:methodology}.1, each 
presence pattern represents a group of mobile entities entering and leaving a 
space simultaneously.
We denote a pattern by the tuple $\langle \set{GS}, NG, t_{\ts{start}}, \tau 
\rangle $ where $\set{GS}$ is a set of mobile entity IDs of size $NG$.
Parameter $t_{\ts{start}}$ is the start time of the pattern, and $\tau$ is its duration.
We assume that each mobile entity generates a detection record at times 
$t_{\ts{start}} + k$ for $0 \leq k < \tau$. A virtual space thus represents an 
actual space, such as a coffee corner, a class room, and so on, for which we 
assume that a fingerprint is known.

\textbf{Generating instances of spaces:} From each virtual space, $NI$ number of instances are generated which will represent the presence patterns of the same space over multiple epochs of length $FD$ with a modified situation.
These instances are generated by varying different sensitivity test parameters as explained later.

\textbf{Generating the mobility dataset:}
Note that each pattern implicitly defines a set of detections.
Each mobile entity $\id{d} \in \set{GS}$ is assumed to generate detections at times $t_{\ts{start}}, t_{\ts{start}}+1, \ldots$.
As a consequence, $\langle \set{GS}, NG, t_{\ts{start}}, \tau \rangle$ for a space $\id{s}$ gives rise to a set of detections $\set{DT}(\id{s},\set{GS}) = \{\langle \id{d},\id{s},t_{\ts{start}}\! +\! k \rangle | \id{d}\! \in\! \set{GS}, 0\! \leq\! k\! <\! \tau\}$.
We construct a dataset by taking the union of sets $\set{\overline{DT}}(\id{s},\set{GS})$ for patterns generated for \id{s}.

\textbf{Clustering:} Each set of detections $\set{\overline{DT}(\id{s})}$ is vectorized using Algorithm~\ref{algorithm2} with a precomputed pair of $FD$ and $FR$ and the accuracy of clustering fingerprint instances to their correct cluster is presented.
For the input $K$ of the \textit{K-means} algorithm, we use the number of original fingerprints as the number of clusters.
The success of the algorithm in clustering is finding $NS$ distinct clusters by mapping the instances of the same space to the same cluster.

\subsubsection{Sensitivity test parameters}

Our goal is to test the clustering accuracy of the fingerprinting technique.
There are in general two groups of parameters that affect the quality of clustering.
The first group represents the inherent uncertainty present in presence patterns.
That is, in real-world settings it is unlikely that a presence pattern repeats itself exactly the same way.
The other group represents the noise introduced by data-collection instruments, such as, for example, missing detections due to collision.
Below we specify how we apply the effects of these parameters on the synthetic dataset.

\textbf{Mobility related sensitivity parameters}

\begin{itemize}

\item{\textbf{Variable start and duration}}: We modify the start and duration of each presence pattern by $t_{\ts{start}}^* \in N(t_{\ts{start}},\tau \alpha_{\ts{ts}})$ and $\tau^*\in N( \tau, \tau\alpha_{\ts{td}})$ such that $t_{\ts{start}}^*+\tau^* <FD$.
  $N(\mu,\sigma)$ represents a normal distribution with mean $\mu$ and standard deviation $\sigma$.

\item{\textbf{Variable group size}}: We modify the set of mobile entity IDs of each presence pattern to $\set{GS}^*$ with a new size $NG^* \in N( NG, NG\alpha_{\ts{gs}})$.

\item{\textbf{New random patterns}}: For each space, we generate $\beta NP$ number of new random patterns with the same procedure that we generated the presence patterns.

\item{\textbf{Removal of patterns}}: We randomly remove $\gamma NP$ number of patterns from the original patterns and create a mobility dataset.
	
\end{itemize}

\textbf{Instrument related sensitivity parameters}
\begin{itemize}

\item{\textbf{Asynchronous detection frequency}}: In reality, the frequency of detections is very much device dependent.
  In order to show the effect of asynchronous detections being sent by mobile entities, we randomly choose $\eta NG$ number of mobile entities from each presence pattern and change their detection period by assigning a random number in the range $[2,0.5\tau]$.

\item{\textbf{Missing detections}}: After creating the detection dataset 
$\set{\overline{DT}}{(\id{s})}$, we randomly remove $\rho$ percent of mobile 
entity IDs for each moment in $\set{\overline{DT}}(\id{s})$.
 (Recall that detections occur at discrete moments in time.)

\end{itemize}
 
 \begin{figure}
 	\centering
 	\begin{subfigure}{0.23\textwidth}
 		\includegraphics[width=\textwidth]{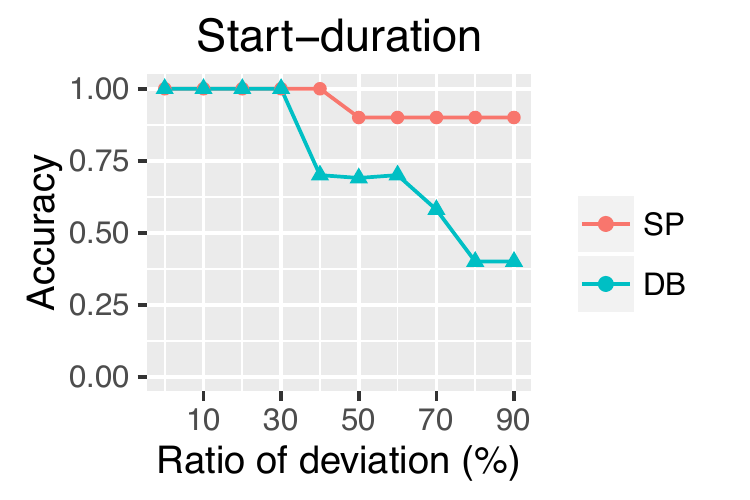}
 		\caption{}
 	\end{subfigure}
 	\begin{subfigure}{0.23\textwidth}
 		\includegraphics[width=\textwidth]{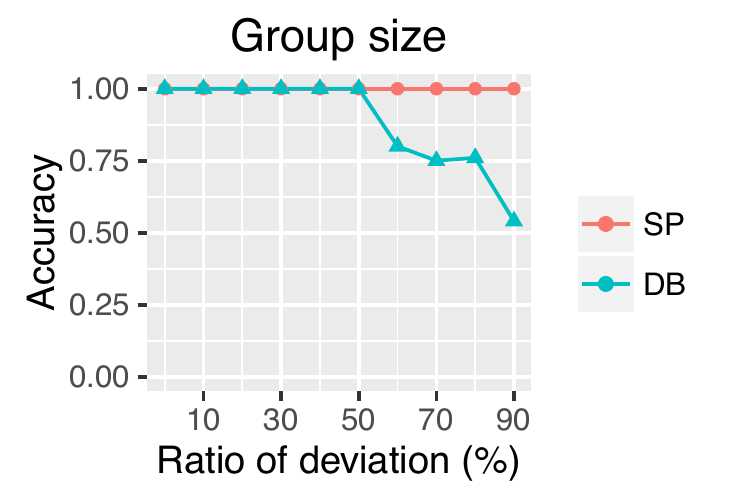}
 		\caption{}
 	\end{subfigure}
 	
 	\begin{subfigure}{0.23\textwidth}
 		\includegraphics[width=\textwidth]{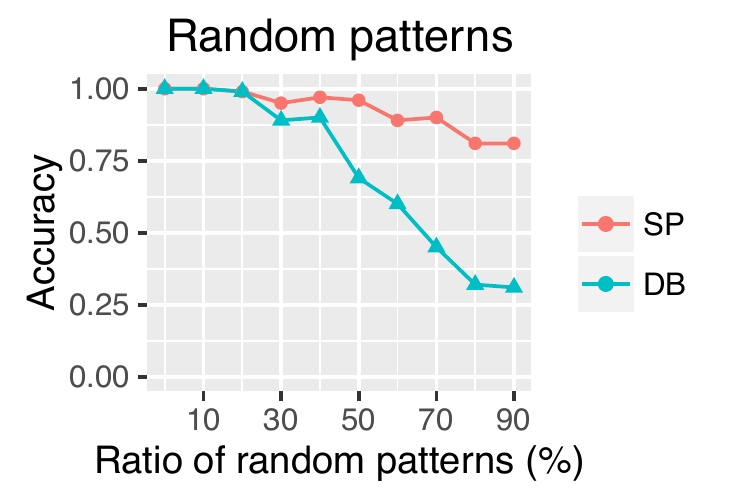}
 		\caption{}
 	\end{subfigure}	
 	\begin{subfigure}{0.23\textwidth}
 		\includegraphics[width=\textwidth]{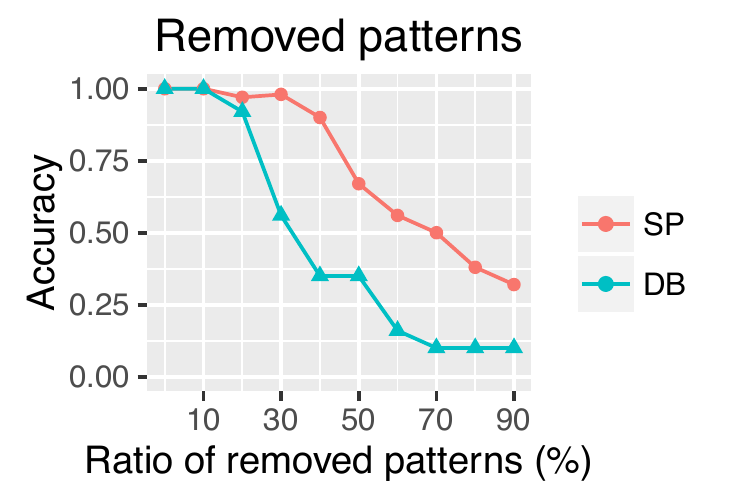}
 		\caption{}
 	\end{subfigure}
 	
 	\begin{subfigure}{0.23\textwidth}
 		\includegraphics[width=\textwidth]{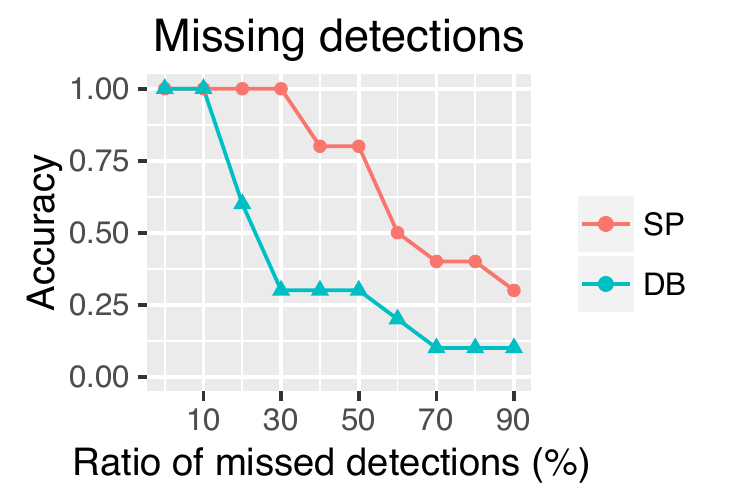}
 		\caption{}
 	\end{subfigure}
 	\begin{subfigure}{0.23\textwidth}
 		\includegraphics[width=\textwidth]{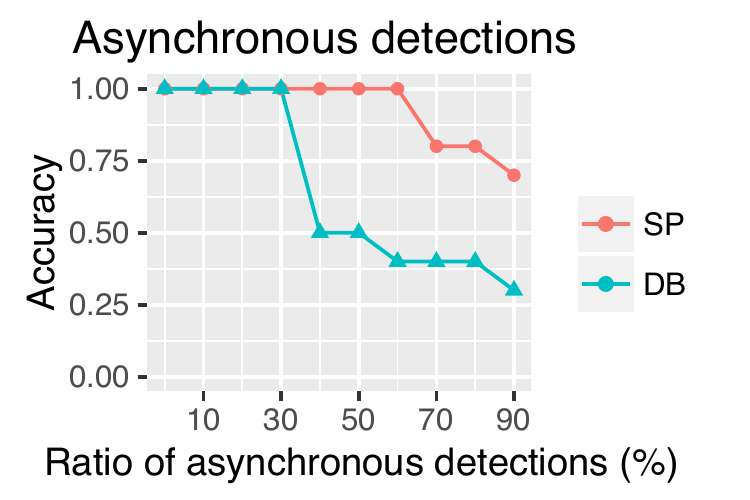}
 		\caption{}
 	\end{subfigure}
 	\caption{Tests with synthetic dataset. $SP$ and $DB$ denote use 
 		of feature 
 		vectors extracted based on \spaceprint, and 
 		\emph{density-based} approach, respectively. }
 	\label{syntheticgraphs}
 \end{figure}

\begin{table}[!ht]
  \caption{\label{tbl:synthetictable} The parameters chosen for generating a 
  synthetic dataset}
  \begin{center}
	\begin{tabular}{cl}\hline
	  \textbf{Parameter} & \textbf{Value} \\ \hline
	  $NS$       &  10 \\
	  $NI$       &  100 \\
	  $FD$   &  1440 \\
	  $FR$   &  60 \\
	  $NG, NP$       &  $ \in [1,100],  \in \mathbb{N}$ \\
	  $t_{\ts{start}}$ &  $\in [1,1440], \in \mathbb{N}$ \\
	  $\tau$     &  $  \in [1,1440-t_{\ts{start}}], \in \mathbb{N}$ \\
	  $\alpha_{\ts{gs}},\alpha_{\ts{ts}},\alpha_{\ts{td}}, 
	  \beta,\gamma,\eta,\rho$ & $\in 
	  [0,0.9], \in \mathbb{Q}$ \\ \hline
	\end{tabular}
  \end{center}
\end{table}

Table~\ref{tbl:synthetictable} represents the parameter ranges used for the 
tests in this section.
The results of analysis with the synthetic dataset are shown in Figure~\ref{syntheticgraphs}.
We use detections from a total of 10 different clusters.  In each figure, we show the accuracy of assigning instances to the correct original cluster while varying a specific sensitivity test parameter.
We note that with 10 clusters, simply assigning all instances to one cluster will lead to 10\% accuracy.
Therefore, an accuracy less than 10\% is meaningless.
In order to have a feeling of how good the accuracy of \spaceprint\ is, we compare it with a \emph{density-based} approach.
We extracted the feature vectors for \spaceprint\ using Algorithm~\ref{algorithm2} and an equivalent feature vector for the \emph{density-based} approach with the fingerprint parameters ($FD= 1440, FR=60$).
The features extracted using these two methods are alternatively used as input to \emph{K-means}.
In the case of \spaceprint, the distance metric introduced in Definition~\ref{def:distance} is used.
For the density-based alternative we use the Euclidean distance.

The graphs presented in Figure~\ref{syntheticgraphs} suggest that using the 
feature vectors extracted by \spaceprint\ results in a higher accuracy than 
using \densitybased\ feature vectors.
Figures~\ref{syntheticgraphs}(a) and (b) show that the accuracy of \spaceprint\ is hardly affected by the changes in start, duration, and group size of random patterns.
It is also seen in Figure~\ref{syntheticgraphs}(c) that introducing new random patterns will not degrade the accuracy of \spaceprint\ as the fixed underlying patterns are being reflected in various elements of the feature vector.
By removing patterns that construct the original space from a generated instance of that space, the accuracy of \spaceprint\ degrades.
However, \spaceprint\ is still much more robust in response to such changes than the \densitybased\ approach (Figure~\ref{syntheticgraphs}(d)).
We see in Figures~\ref{syntheticgraphs}(e) and (f) that \spaceprint\ is also 
more robust to the noise introduced by instrument-related parameters than the 
density-based approach.
Although missing detections and variable frequency of detections will distort parts of the feature vector representing presence patterns with a finer period, the effect of patterns will still be present in elements which represent coarser sampling periods.

\subsection{Real datasets}

In this section, we apply our fingerprinting framework on two datasets 
collected from real-world public spaces. Both of these datasets 
conform to our model in Section \ref{sec:system-model}. 
However, due to having different data collection mechanisms, 
they 
have subtle differences in terms of sparsity of detections and variety of 
spaces (summarized in Table \ref{tbl:datasets}).
The first dataset is a set of WiFi detections very rich in terms of the number 
of detections collected per space but contains data from a limited number of 
spaces. This dataset is collected by WiFi scanners placed in all \textbf{coffee 
corners} at our university campus for a period of 5 months\footnote{Anonymous 
WiFi scanning can be performed by hashing MAC addresses on the fly and 
providing an opt-out option for visitors.}.
The second one, which is a dataset of Foursquare \cite{yang2013fine} check-ins, 
is very rich in terms of diversity of spaces while being much sparser in terms 
of the number of detections available per location. We chose locations within 
the top 100 location categories with data from more than 531 days.

\begin{table}[!ht]
	\caption{\label{tbl:datasets} Datasets}
	\begin{center}
		\begin{tabular}{ccl}\hline
		&\textbf{WiFi DB}& \textbf{Foursquare DB}\\ 
			\hline
			\textbf{\#Spaces} &8 & 10,000\\
			\textbf{ 
				\#Categories}& 1 & 100 \\
			\textbf{\#Mobile entities} &700,000  & 201,132 
			\\
			\textbf{Duration (days)} & 150  & 531 \\
			\textbf{\#Detections} & 95,000,000  & 24,474,738\\ 
			\textbf{\#Detections per space per day} & \textbf{79,166} & 
			\textbf{2.3}
			\\\hline
		\end{tabular}
	\end{center}
\end{table}

\subsection {Case study with WiFi dataset}

In what follows, we demonstrate the procedure of extracting fingerprinting parameters and feature vectors using the WiFi dataset.

\subsubsection{Extracting fingerprint parameters} 

In order to calculate the feature vectors, it is required that the optimal 
fingerprinting parameters, $FD$ and $FR$, are extracted for each space 
separately.
We show how we find these values for one of the coffee corners using 
Algorithm~\ref{algorithm3}.
Figures~\ref{parameters}(a) and (b) illustrate how the optimal fingerprint duration can be extracted.
What is shown in these graphs is the average pairwise distance of feature vectors calculated by varying the parameter, $FD$.
It should be noted that the comparison of fingerprint durations is only fair if it is performed based on the pairwise distance of vectors of the same length (vectors of longer length will have more elements equal to zero and therefore, their distance will be smaller).
To have feature vectors of the same length, we changed the fingerprinting resolution, $FR$, based on the fingerprint duration such that the size of the resulting feature vector stays constant.
This is achieved by setting $\frac{FD}{FR}$ to a constant value.
We calculated these distances for vectors of length 17186 and 791 features, respectively.
The optimal fingerprint duration is the one that \emph{minimizes} the distance between two feature vectors, and thus maximizing similarity.
For both graphs shown in Figure~\ref{parameters}(a) and (b), this value is acquired at a duration equivalent to one week (168 hours).
In Figure~\ref{parameters}(b), it is also seen that a fingerprint duration equivalent to 3 days is the worst fingerprinting choice, leading to maximum dissimilarity between vectors.

\begin{figure}
	\centering
	\begin{subfigure}{0.23\textwidth}
		\includegraphics[width=\textwidth]{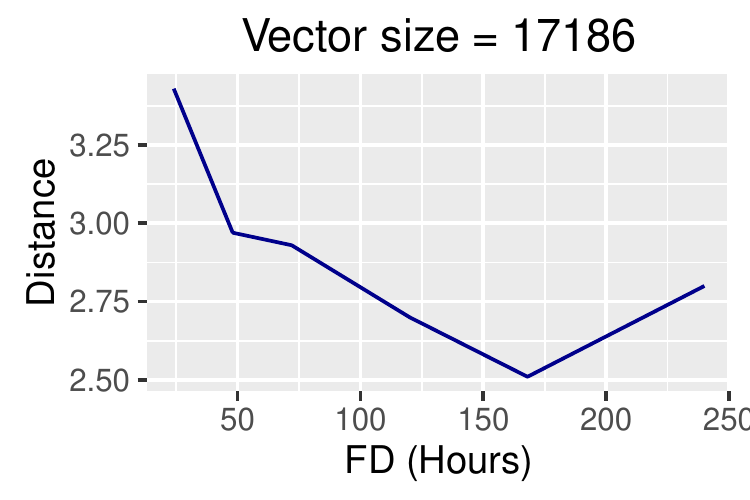}
		\caption{}
	\end{subfigure}
	\begin{subfigure}{0.23\textwidth}
		\includegraphics[width=\textwidth]{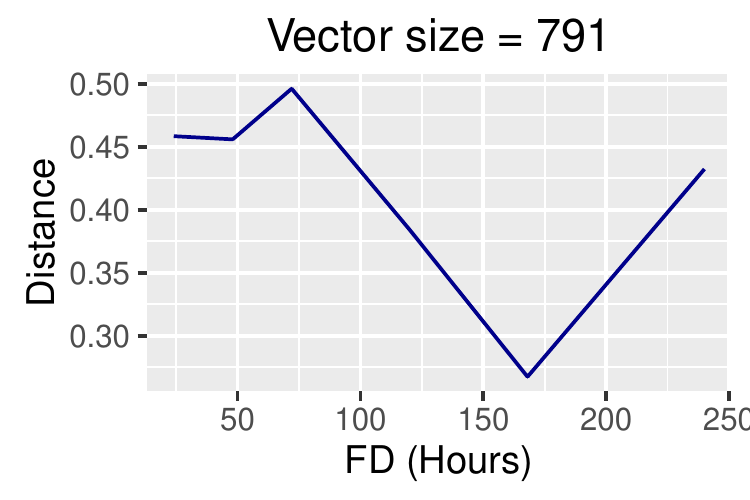}	
		\caption{}
	\end{subfigure}
	\begin{subfigure}{0.23\textwidth}
		\includegraphics[width=\textwidth]{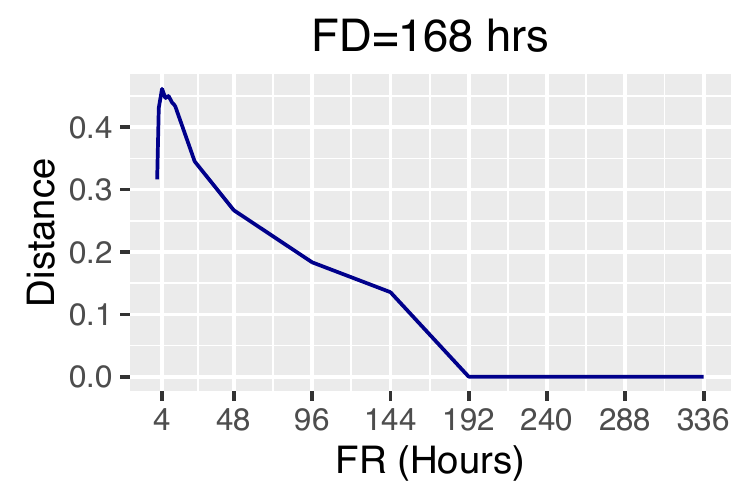}
		\caption{}
	\end{subfigure}
	\begin{subfigure}{0.23\textwidth}
		\includegraphics[width=\textwidth]{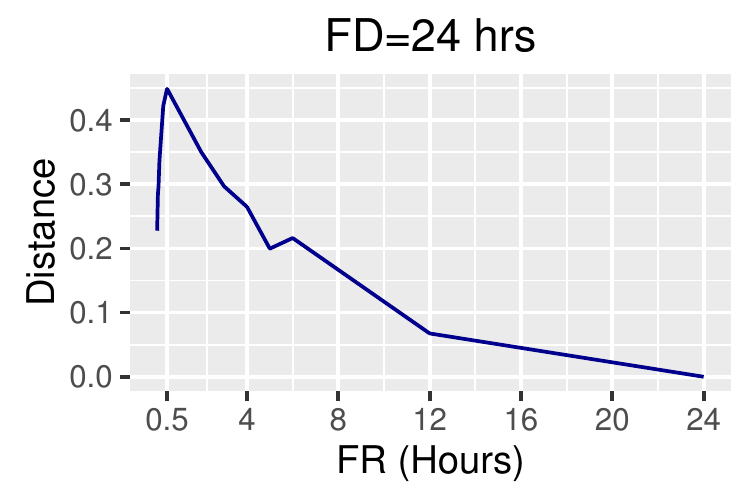}
		\caption{}
	\end{subfigure}
	\caption{(a-b) Choosing optimal fingerprinting duration (c-d) choosing 
	optimal fingerprinting resolution}
	\label{parameters}
\end{figure}
Figures~\ref{parameters}(c) and (d) show how the optimal fingerprint resolution can be chosen.
The optimal fingerprint resolution is the one that \emph{maximizes} the distance between feature vectors revealing more detail about the space.
We have looked at the optimal resolution when the fingerprint duration is equal to the optimal fingerprint duration (1 week time).
The results suggest that a resolution of 4 hours can still reveal the differences between feature vectors.
As most of the weekdays are similar, we also looked at the spaces (only over weekdays) with a fingerprinting duration of 24 hours.
The figures suggest that a resolution of 30 minutes suffices to reveal the necessary level of detail when the fingerprint is only extracted from weekdays.
This is in fact the minimal resolution that still captures detections from static devices.
Any finer resolution will result only in more zero-valued entries in feature vectors.
Note that deriving two optimal resolutions does not contradict Lemma~\ref{lemma:fr}, as the daily resolution is extracted only from weekdays.

\subsubsection{Two-dimensional representation of feature vectors:}
To further see how \emph{Spaceprint} represents the similarities between the situation in these coffee corners, we also visualize the extracted feature vectors from the whole dataset in Figures~\ref{resolutiongraphweek} and \ref{resolutiongraphday}.
The feature vectors extracted have $n$ elements (e.g., with $FD=168$ hours and $FR=1$ hour, $n = 23355$) and can be represented as points in an $n$-dimensional coordinate system.
In order to represent such points, we map them to a two-dimensional space using multi-dimensional scaling~\cite{borg2005modern}.
%
This method takes a dissimilarity matrix composed of the pair-wise distance between all vectors.
By applying principal component analysis on such a matrix a coordinate matrix is generated whose configuration minimizes a loss function.
Using the dissimilarity matrix calculated based on the distance function from Definition~\ref{def:distance}, multi-dimensional scaling can capture the effects of the nonuniform size of the elements in our feature vectors.

\begin{figure}
	\centering
	\begin{subfigure}[b]{0.23\textwidth}
		\includegraphics[width=\textwidth]{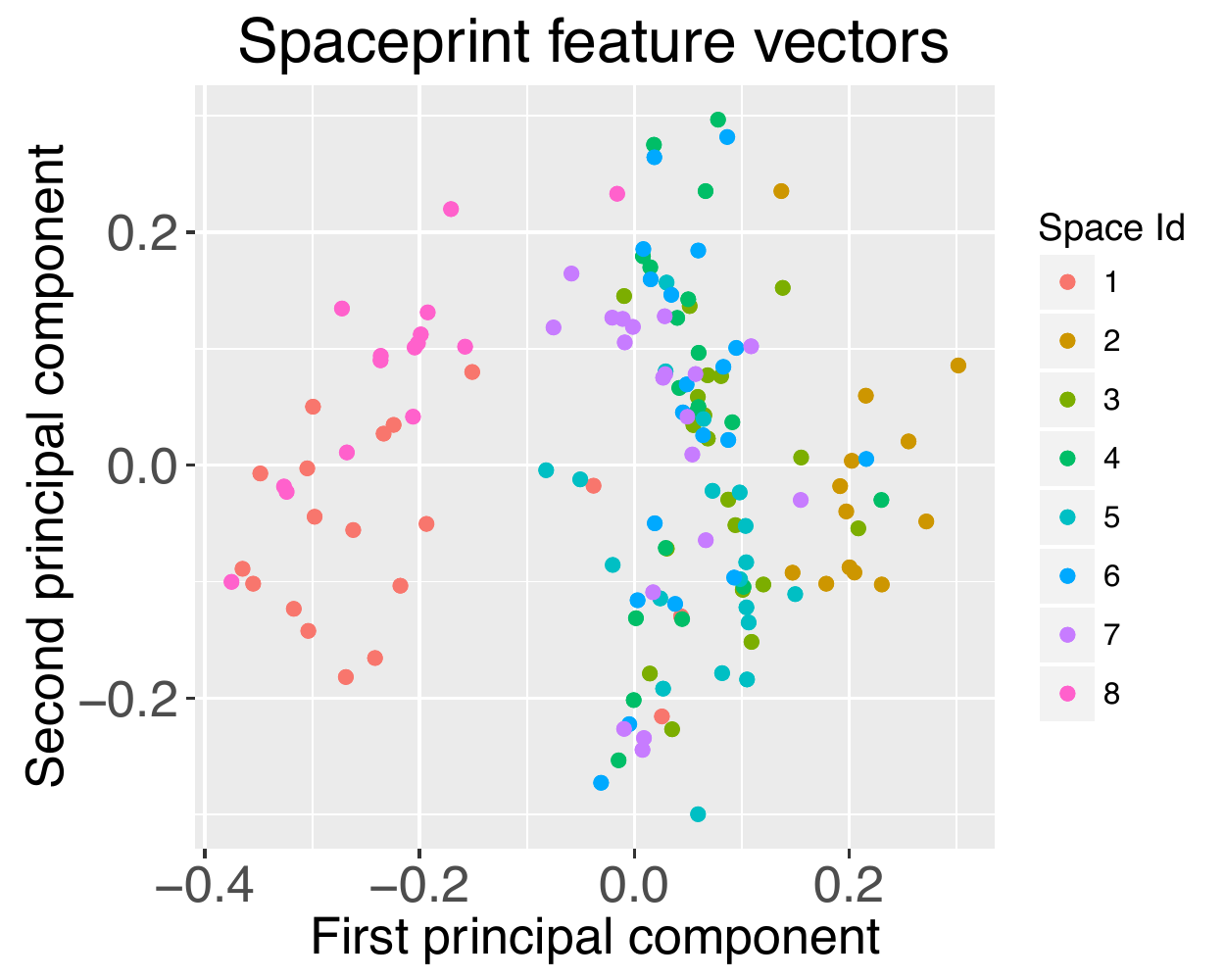}
	\end{subfigure} 
	\begin{subfigure}[b]{0.23\textwidth}
		\includegraphics[width=\textwidth]{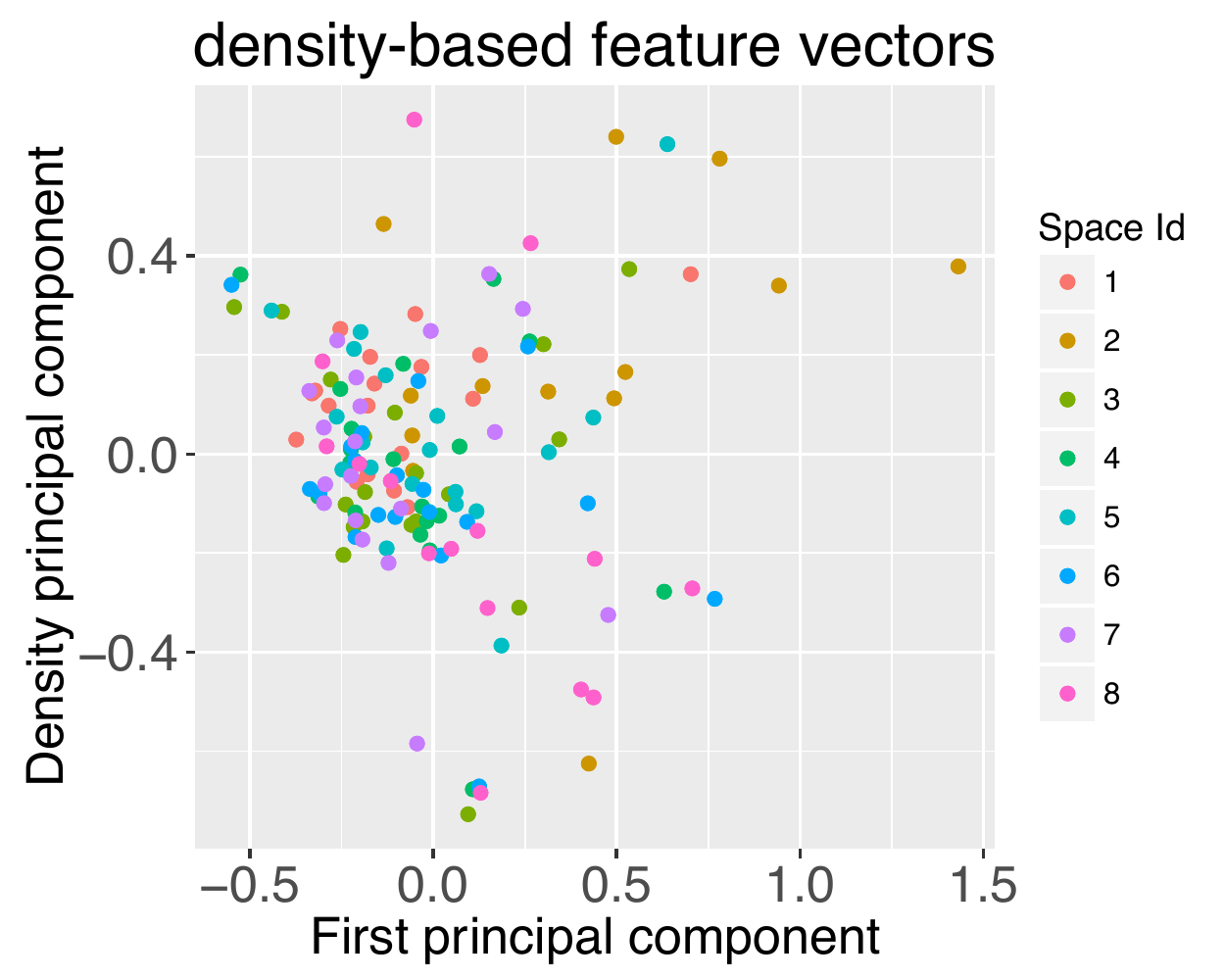}
		
	\end{subfigure}
	\caption{Two-dimensional representation of feature vectors 
	of 
	\textit{Spaceprint} and 
	\textit{density-based}
	approach. Each 
		point represents one week of data. $FD=168$ hours and $FR=1$ 
		hour.}
	\label{resolutiongraphweek}
\end{figure}

\begin{figure}
	\centering
	\begin{subfigure}[b]{0.23\textwidth}
		
		\includegraphics[width=\textwidth]{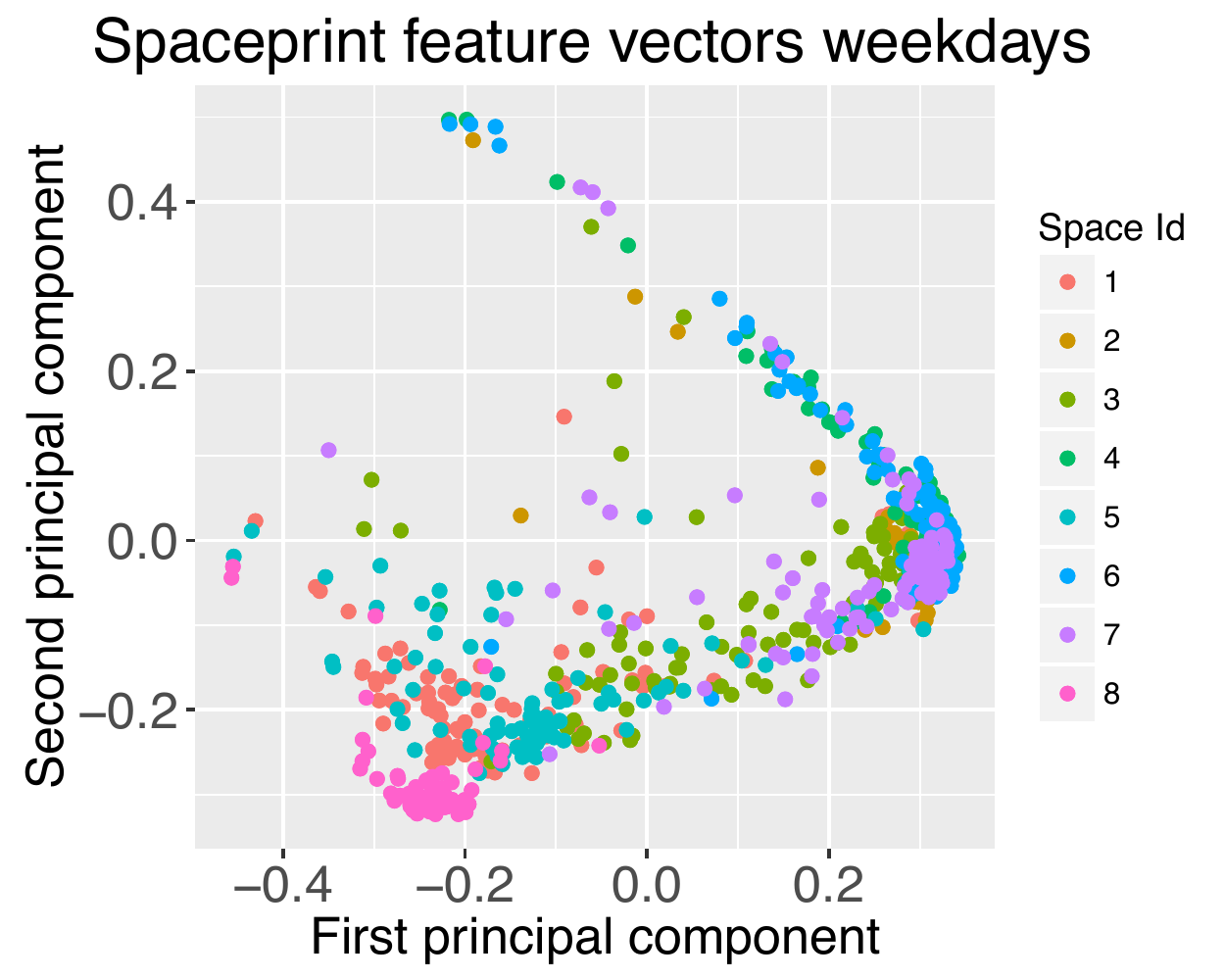}
		\caption{}	
	\end{subfigure} 
	\begin{subfigure}[b]{0.23\textwidth}
		\includegraphics[width=\textwidth]{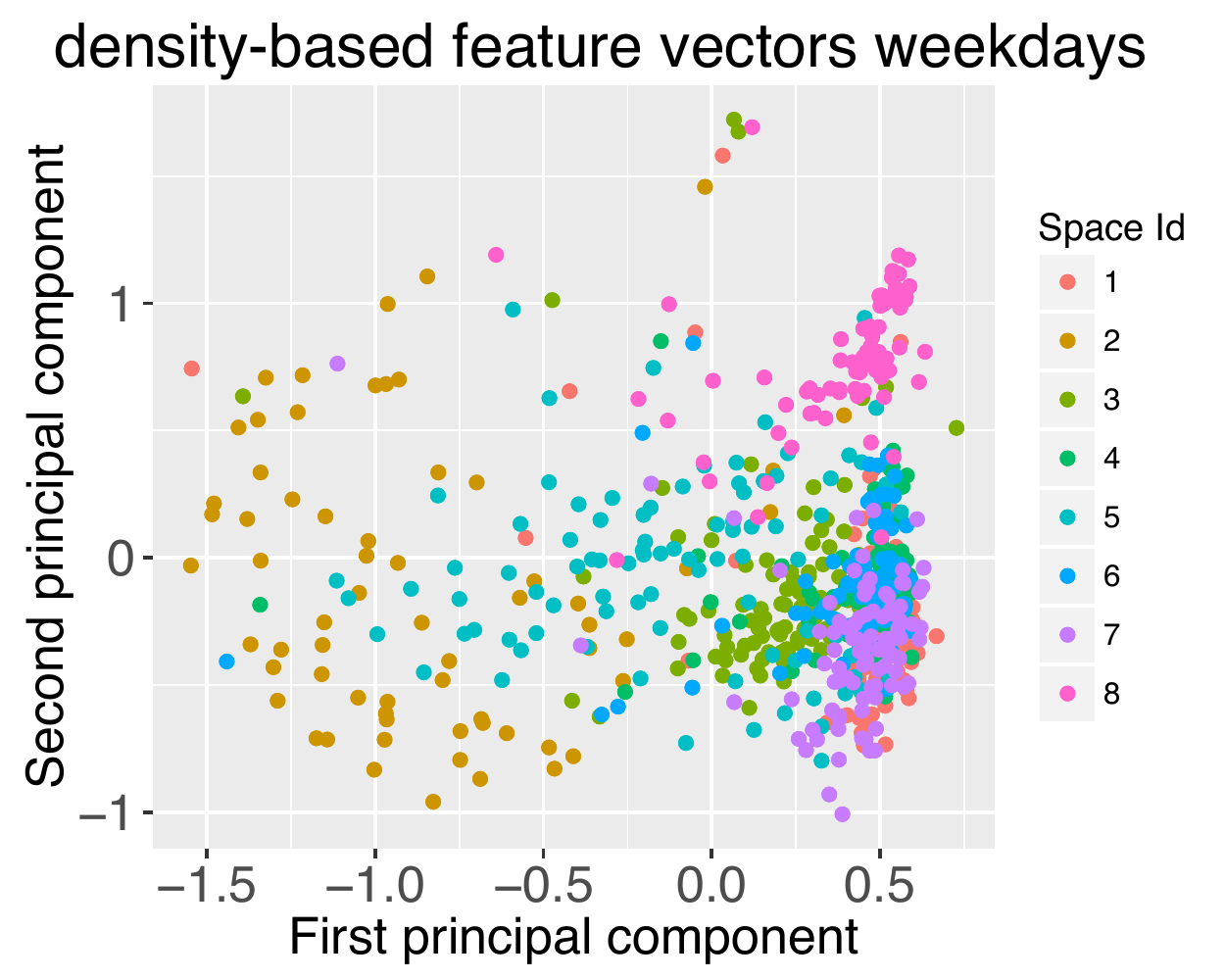}
		\caption{}		
	\end{subfigure}
	\begin{subfigure}[b]{0.23\textwidth}
		\includegraphics[width=\textwidth]{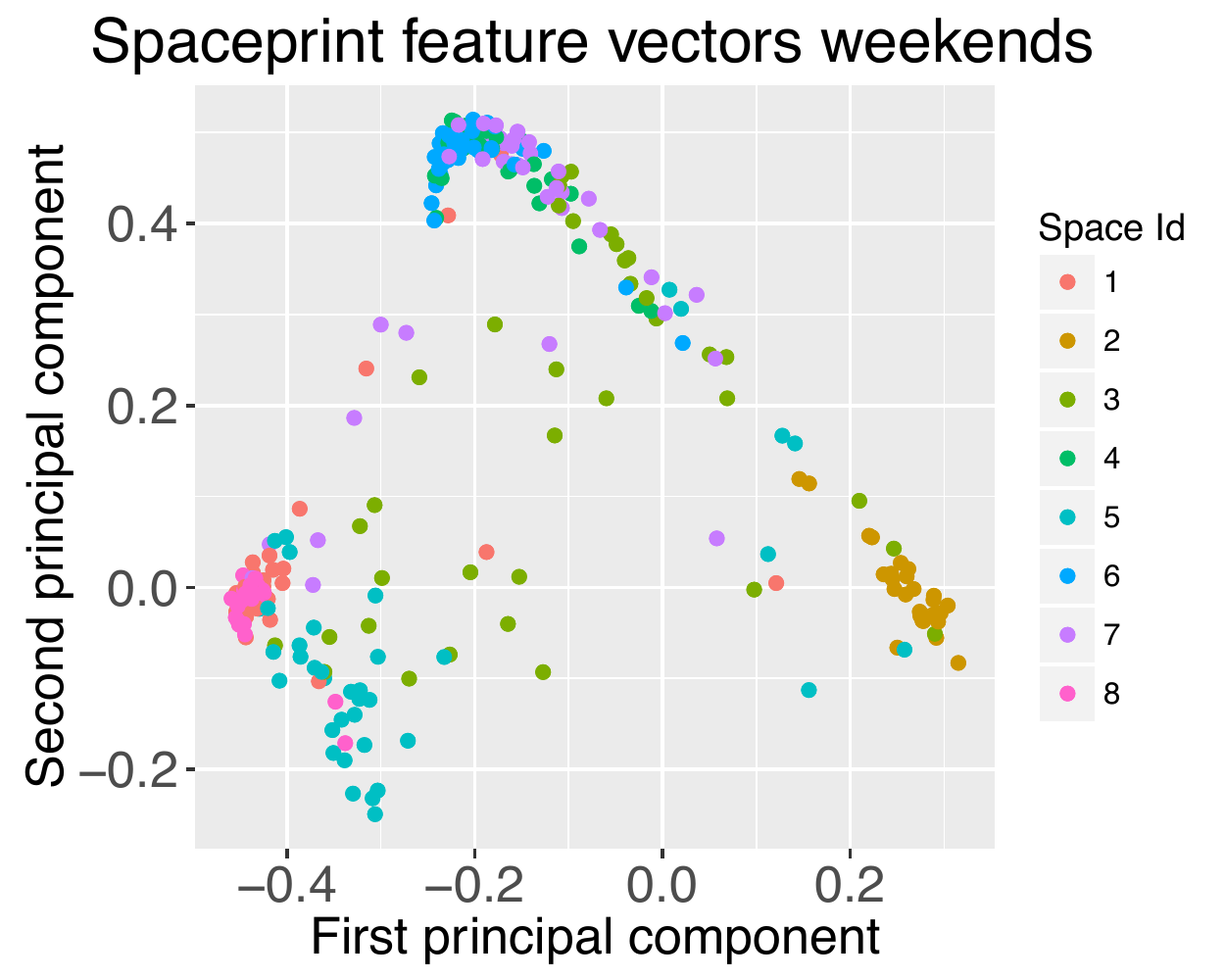}	
		\caption{}	
	\end{subfigure} 
	\begin{subfigure}[b]{0.23\textwidth}
		\includegraphics[width=\textwidth]{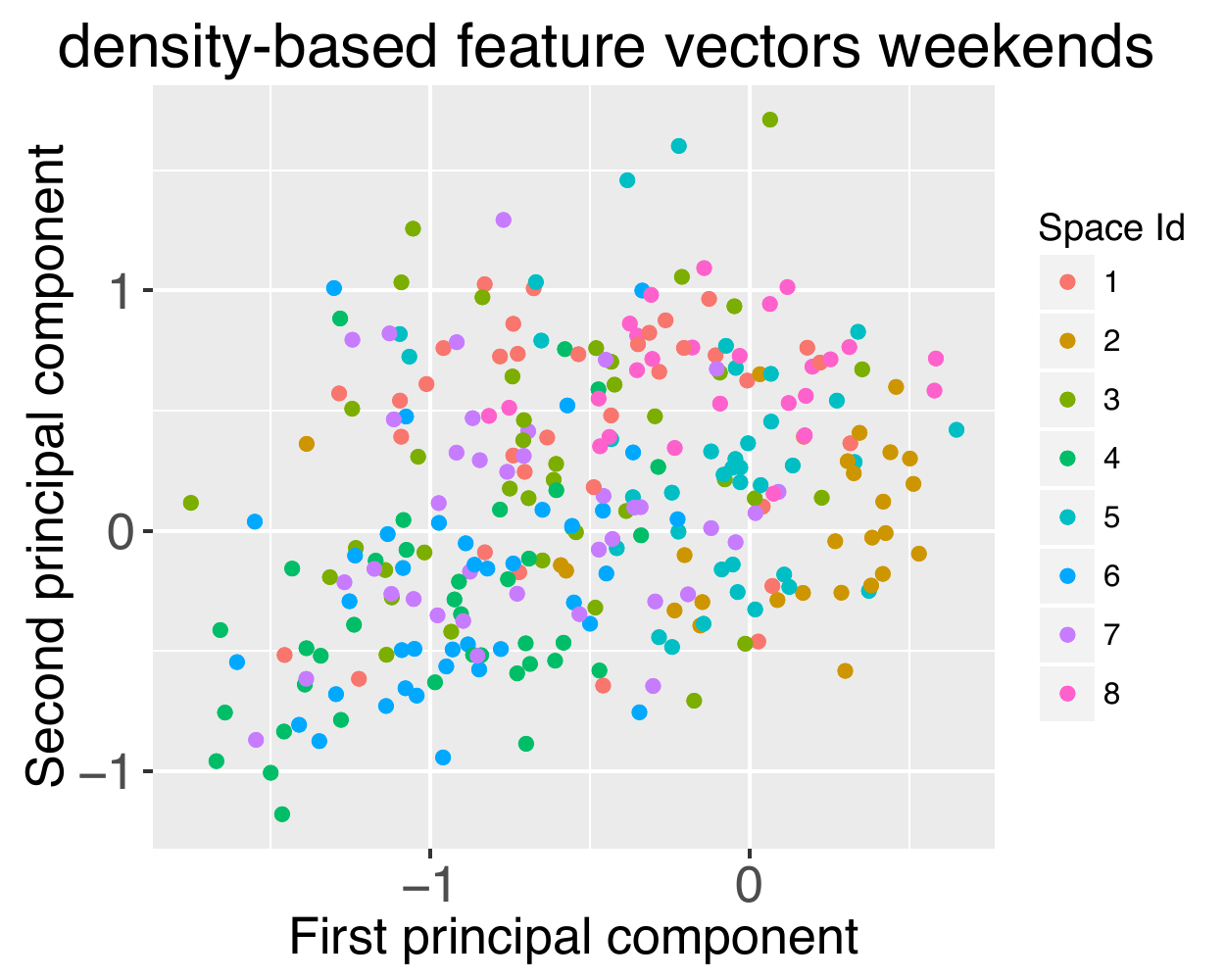}
		\caption{}	
	\end{subfigure}
	
	\caption{Two-dimensional representation of feature vectors 
	of 
		\textit{Spaceprint} and \textit{density-based} 
		approach. Each 
		point represents one day of data. $FD=24$ hours and $FR=1$ 
		hour.}
	\label{resolutiongraphday}
\end{figure}
The results are presented in Figures~\ref{resolutiongraphweek} and \ref{resolutiongraphday}.
We compare the result of vectorizing using \spaceprint\ and the \emph{density-based} approach.
In Figure~\ref{resolutiongraphweek}, we have vectorized each week of data ($FD=168$ hours and $FR=1$ hour).
As seen, \spaceprint\ results in a clearer distinction between points of the same color.
In other words, the identity of the location is reflected in the similarity 
between weeks of data from the same space. In Figure~\ref{resolutiongraphday}, 
using the parameters $FD=24$ hours and $FR=1$ hour, each day is vectorized 
separately.  We also present the weekdays and weekends in separate graphs.
Again, \spaceprint\ provides a better distinction between the situation in spaces by placing points representing days in different spaces further from each other.
This is specifically visible in the case of weekends (Figure 
\ref{resolutiongraphday}(c)-(d)).
The data presented here includes occasional changes in normal presence patterns, due to holidays, special events such as conferences, etc.
Therefore, there are naturally outliers, yet the identity of locations is evident.

\subsubsection{Clustering performance (Latent variable of identity)}
To further evaluate how such feature vectors can be used to create a unique 
fingerprint for spaces, we cluster them using \emph{K-means} algorithm. The 
goal is to see if we can distinguish from \emph{which} space they have been 
extracted. Each 
space in this dataset has a \textbf{space id}. We cluster feature vectors 
extracted from 150 days and look for 8 different clusters representing 8 
different space ids.
This is equivalent of assigning points of the same color (in Figure 
\ref{resolutiongraphday}) to the same cluster.
Performance of the clustering task in terms of Accuracy, Random Index, 
F-measure, and Normalized Mutual Information (NMI) is presented in Figure 
\ref{coffeecorneraccuracy}.
As seen, the results are in favor of \emph{Spaceprint} for all of these indicators.

\begin{figure}
	\centering
	\begin{subfigure}{0.4\textwidth}
		\includegraphics[width=\textwidth]{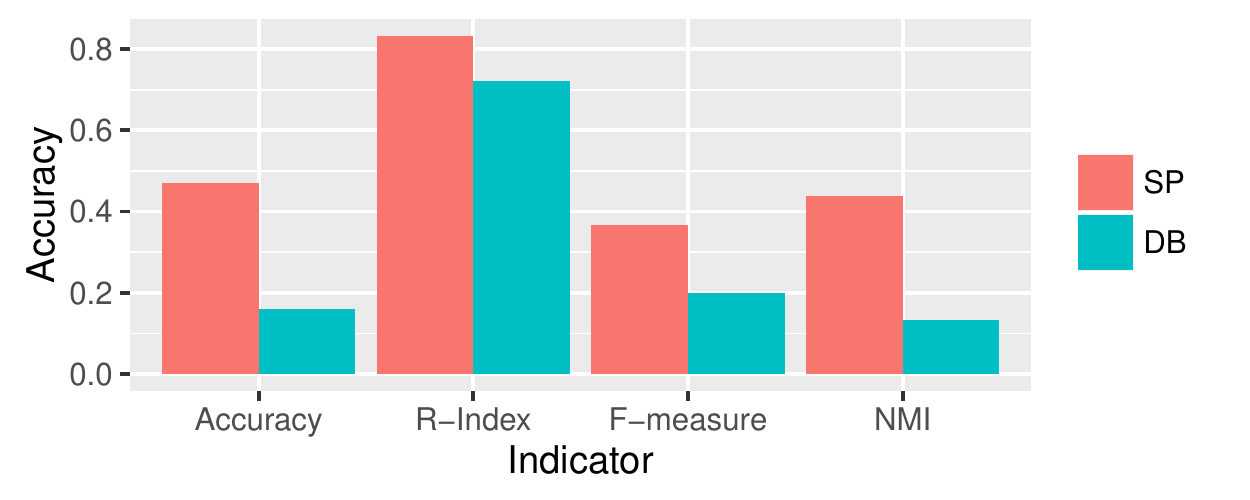}
	\end{subfigure}
	\caption{Performance of clustering for $FD=168$ hours. $SP$ and 
	$DB$ denote use of \textit{Spaceprint} and \textit{density-based} features.}
	\label{coffeecorneraccuracy}
\end{figure}
\subsection{Case study with the Foursquare dataset}

In this section, we perform evaluations on a dataset collected from Foursquare location-based social network.
Each space in this dataset has a \textbf{space id} and a \textbf{space 
category}.
Taking each of these two labels as ground truth for determining the clustering 
performance, gives us the opportunity to perform two types of evaluations.
The first one, similar to evaluations on the WiFi dataset, is to classify 
feature vectors to know from \textit{which} space they were extracted.
The second one, is to classify feature vectors of a group of spaces to know 
from \textit{what type} of space they were collected.
Performance is evaluated based on classification of spaces with category labels 
such as \textit{home}, \textit{office}, \textit{airport}, \textit{restaurant}, 
\textit{Chinese restaurant}, \textit{road}, etc.
(Full list is omitted due to lack of space).

\subsubsection{Clustering performance (Latent variable of identity)}
For the first experiment, performance of \emph{Spaceprint} and 
\emph{density-based} method ($SP_i$ and $DB_i$) is compared based on 
classification between $K$ randomly chosen space ids $(K \in [2,10])$ and 
feature vectors extracted from 531 days.
The accuracy of clustering algorithm is calculated on correctly clustering feature vectors of different spaces based on their original space id.

\begin{figure}
	\centering
	\begin{subfigure}{0.23\textwidth}
		\includegraphics[width=\textwidth]{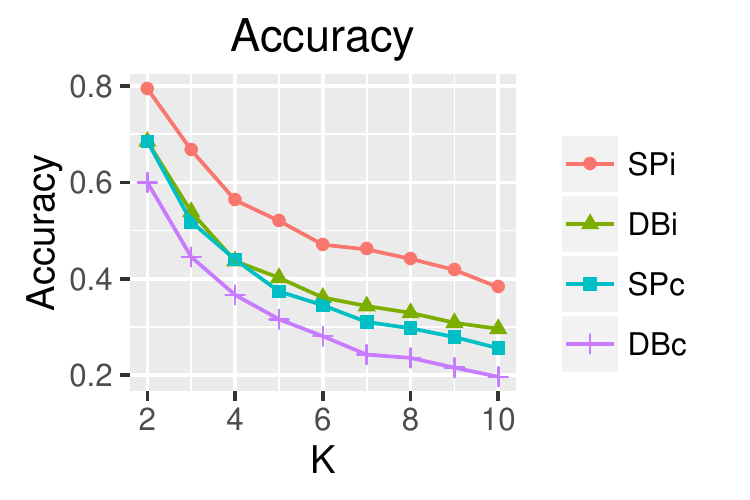}
		\caption{}
	\end{subfigure}
	\begin{subfigure}{0.23\textwidth}
		\includegraphics[width=\textwidth]{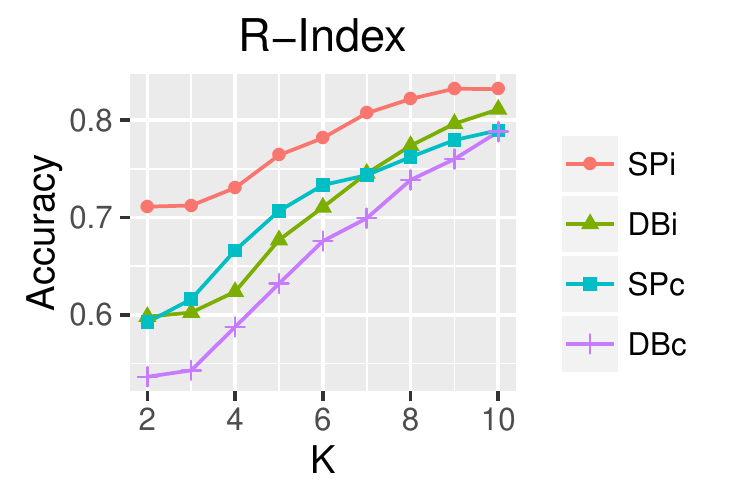}
		\caption{}
	\end{subfigure}
	
	\begin{subfigure}{0.23\textwidth}
		\includegraphics[width=\textwidth]{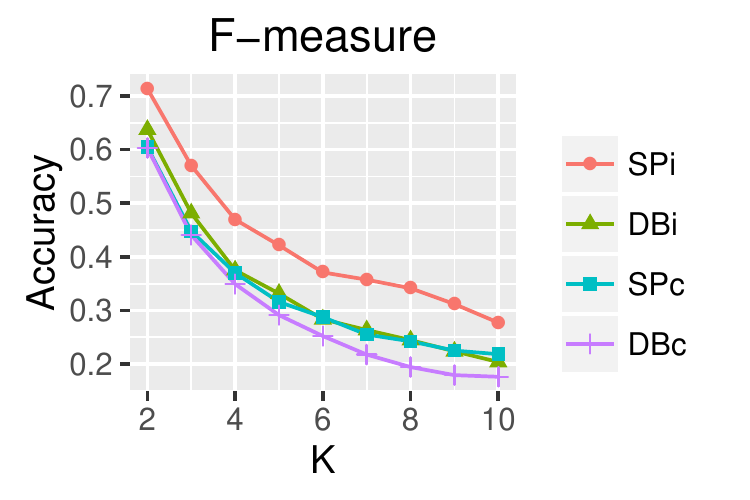}
		\caption{}
	\end{subfigure}
	\begin{subfigure}{0.23\textwidth}
		\includegraphics[width=\textwidth]{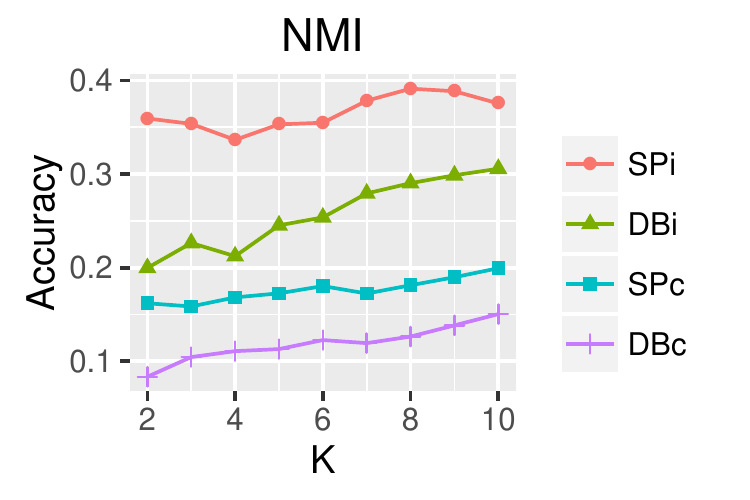}
		\caption{}
	\end{subfigure}
	\caption{Tests with Foursquare dataset. $SP$ and $DB$ denote use 
		of feature 
		vectors extracted based on \spaceprint, and 
		\emph{density-based} approach, respectively. 
		Subscripts "i" and "c" refer to classification 
		based on the latent variable of identity and 
		category, respectively. $K$ is the number of clusters.}
	\label{foursquare}
\end{figure}

\subsubsection{Clustering performance (Latent variable of category)}
For the second experiment, we chose $K$ randomly chosen categories $(K \in 
[2,10])$ and further selected 10 spaces per category.
We similarly extracted the feature vectors from 531 days.
The accuracy of \emph{Spaceprint} and \emph{density-based} method ($SP_c$ and 
$DB_c$), is compared based on correctly clustering the feature vectors of a 
group of spaces based on their correct category.
The results presented in Figure \ref{foursquare}, are the mean value acquired 
from 100 runs of experiment per $K$ with $FD= 168$ and $FR=1$ hour.
Generally, regardless of the high sparsity level of this dataset, comparisons 
shown in Figure \ref{foursquare} (a)-(d) are in favor of \emph{Spaceprint} for 
both experiments.
Higher performance in terms of NMI shows that even misclassification of spaces 
based on category yields more information about the similarity of spaces in 
different clusters.
An example will be misclassifying a space with the category label of 
\emph{restaurant} to the category of \emph{Chinese restaurant}.

\section{Discussion and conclusions}
\label{sec:conclusion}

In this paper, we presented \textit{Spaceprint}, a technique for creating 
spatial fingerprints for repetitive 
situations in public spaces.
What makes \spaceprint\ unique is its fully 
automatic operation with minimal input from 
anyone who 
operates it.
Our evaluations show that the automated 
fingerprinting of spaces is indeed 
possible, 
opening the path to more sophisticated 
approaches for automated 
situation-awareness. 
We also conclude that \spaceprint\ is relatively insensitive to parameters that 
can degrade the classification accuracy. By 
automatically extracting 
fingerprint 
parameters, \textit{Spaceprint} allows embedding privacy by design in data 
collection by 
anonymizing 
(e.g. 
 hashing) data with timely hashes based on fingerprint duration parameter such 
that 
the accuracy of the spatial fingerprint is 
also not affected.
In this paper, we looked at the possibility 
of 
fingerprinting repetitive situations in a 
single space. 
Our future work 
entails refining this method to consider 
interaction between multiple spaces in 
creating these fingerprints.

\bibliographystyle{ACM-Reference-Format}
\bibliography{bibliography} 

\end{document}